\documentclass[letterpaper]{article} 
\usepackage{aaai2026}  
\usepackage{times}  
\usepackage{helvet}  
\usepackage{courier}  
\usepackage[hyphens]{url}  
\usepackage{graphicx} 
\usepackage{natbib}  
\usepackage{caption} 
\usepackage{dsfont}
\usepackage{amsmath}
\usepackage{booktabs} 
\usepackage{amsthm}
\usepackage{multirow} 
\usepackage{algorithm}
\usepackage{algorithm}
\usepackage{algpseudocode}
\usepackage{amsmath}
\usepackage{amssymb}
\algtext*{EndFor}
\algtext*{EndIf}
\newtheorem{theorem}{Theorem}
\newtheorem{lemma}[theorem]{Lemma}

\newtheorem{definition}{Definition}

\usepackage{newfloat}
\usepackage{listings}
\DeclareCaptionStyle{ruled}{labelfont=normalfont,labelsep=colon,strut=off} 
\floatstyle{ruled}
\newfloat{listing}{tb}{lst}{}
\floatname{listing}{Listing}
\copyrighttext{This is an extended version (pre-print) of the paper accepted to the\\
        \mbox{\texttt{\spaceskip 0.35em
        AAAI-26 Conference on Artificial Intelligence.}}}

\title{MeshA*: Efficient Path Planning With Motion Primitives}
\author {
    Marat Agranovskiy,
    Konstantin Yakovlev
}
\affiliations {
    agrinscience@gmail.com, yakovlev.ks@gmail.com
}

\usepackage{bibentry}

\begin{document}

\maketitle

\begin{abstract}
We study a path planning problem where the possible move actions are represented as a finite set of motion primitives aligned with the grid representation of the environment. That is, each primitive corresponds to a short kinodynamically-feasible motion of an agent and is represented as a sequence of the swept cells of a grid. Typically, heuristic search, i.e. A*, is conducted over the lattice induced by these primitives (lattice-based planning) to find a path. However, due to the large branching factor, such search may be inefficient in practice. To this end, we suggest a novel technique rooted in the idea of searching over the grid cells (as in vanilla A*) simultaneously fitting the possible sequences of the motion primitives into these cells. The resultant algorithm, MeshA*, provably preserves the guarantees on completeness and optimality, on the one hand, and is shown to notably outperform conventional lattice-based planning (x1.5-x2 decrease in the runtime), on the other hand.
\end{abstract}

\begin{links}
    \link{Code}{https://github.com/PathPlanning/MeshAStar}
\end{links}

\section{Introduction}

Kinodynamic path planning is a fundamental problem in AI, automated planning, and robotics. Among the various approaches to tackle this problem, the following two are the most widespread and common: sampling-based planning~\cite{sampling, prims_in_sampling} and lattice-based planning~\cite{pivtoraiko2005, pivtoraiko2009}. The former methods operate in continuous space, rely on the randomized decomposition of the problem into smaller sub-problems, and are especially advantageous in high-dimensional planning (e.g., planning for robotic manipulators). Still, they provide only probabilistic guaranties of completeness and optimality. Lattice-based planners rely on the discretization of the workspace/configuration space and provide strong theoretical guaranties with respect to this discretization. Consequently, they may be preferable when the number of degrees of freedom of the agent is not high, such as in mobile robotics, where one primarily considers the coordinates and the heading of the robot. In this work, we focus on the lattice-based methods for path planning in $(x,y,\theta)$.   

Lattice-based planning methods reason over the so-called \emph{motion primitives} -- the precomputed kinodynamically-feasible motions from which the sought path is constructed -- see Fig.~\ref{example_traj}. Stacked motion primitives form a \textit{state lattice}, i.e., a graph, where the vertices correspond to the states of the agent and the edges correspond to the motion primitives. A shortest path on this graph may be obtained by algorithms, such as A*~\cite{astar}, that guarantee completeness and optimality. Unfortunately, when the number of motion primitives is high (which is not uncommon in practice), searching over the lattice graph becomes computationally burdensome. To this end, in this work, we introduce a novel perspective on lattice-based planning.

\begin{figure}[t]
\centering
\includegraphics[width=0.47\textwidth]{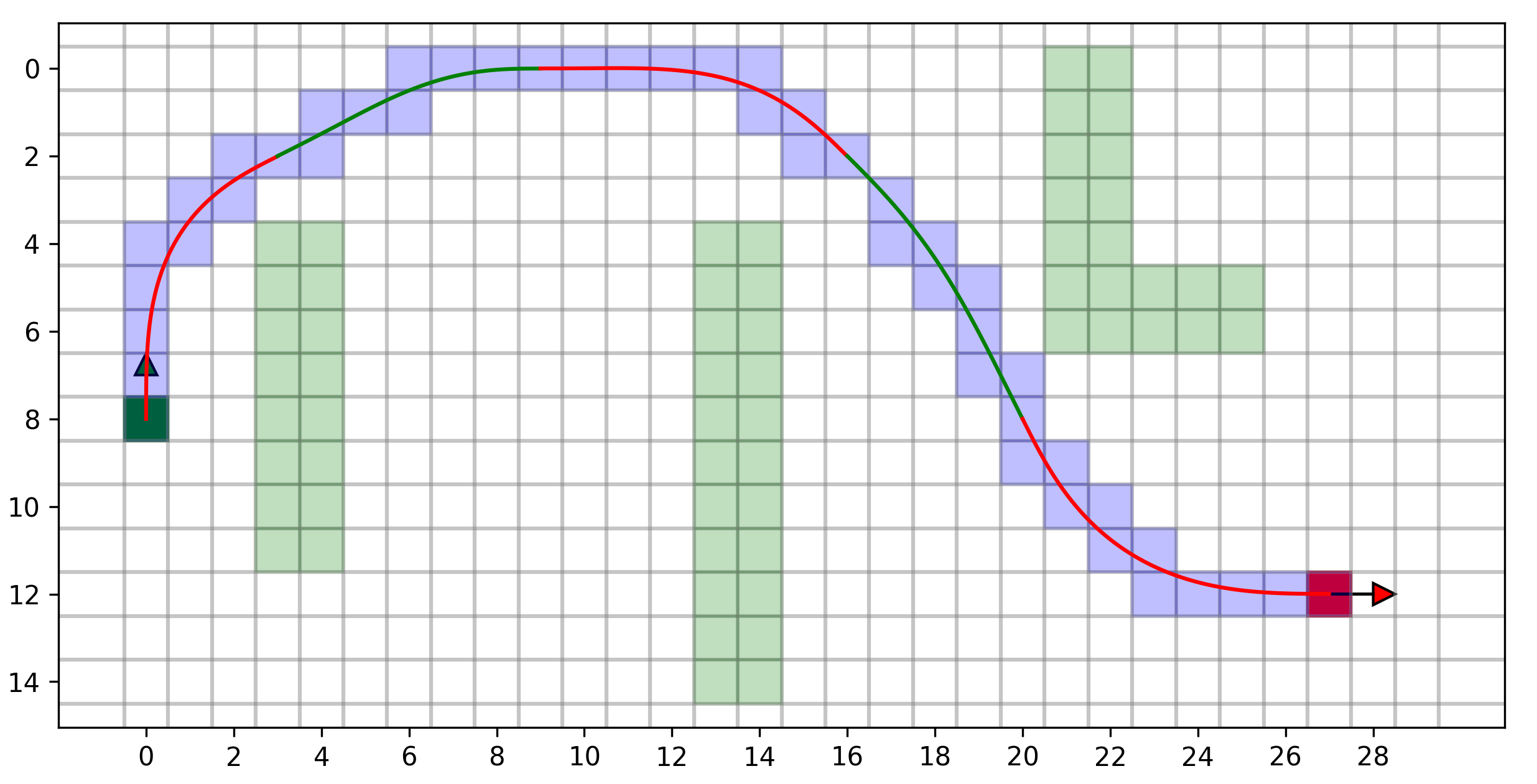}
\caption{Example of the path planning problem. The workspace is discretized to a grid, where the green cells correspond to the obstacles and the white ones represent the free space. The green cell with an arrow denotes the start state (position and heading), while the red cell -- the goal one. The path is composed of the primitives (green and red segments). The swept cells are shown in blue.}
\label{example_traj}
\end{figure}

We leverage the assumption that the workspace is represented as an occupancy grid (a standard practice for path planning) and search over this grid in a cell-by-cell fashion to form a sequence of cells such that a sought path, i.e. a sequence of the motion primitives, may be fitted into this sequence of cells. We introduce a dedicated technique to reason simultaneously about the grid cells and the motion primitives that pass through these cells within the search process. Such reasoning allows us to decrease the branching factor, on the one hand, and to maintain the theoretical guarantees, on the other hand. Empirically, we show that the introduced path planning method, called MeshA*, is notably faster than the conventional A* and its lazy variant (this holds when the weighted heuristic is utilized as well).

\section{Related Work}

Various approaches to pathfinding consider a mobile agent’s kinodynamic constraints~\cite{review}. One prominent group is sampling-based planning. Classical representatives include the RRT algorithm~\cite{rrt}, which rapidly explores the configuration space, its anytime modification RRT*~\cite{rrt_star} aimed at converging to optimal solutions, and variants like Informed RRT*~\cite{informed_rrt}, which incorporates heuristics, RRT-Connect~\cite{rrt_connect} for fast bidirectional planning, and $\mathrm {RRT^{X}}$~\cite{rrt_x} for fast re-planning.

Another direction is lattice-based planning, where a set of motion primitives that respect the constraints of the agent is constructed, and then a suitable sequence of these primitives is sought. Primitives can be generated using B-splines~\cite{b_spline_prims}, the shooting method~\cite{shoot_method_prims}, the covering method~\cite{cover_prims}, learning-based techniques~\cite{learn_prims}, and others. In this work, we follow the lattice-based approach, but assume the primitives are given in advance (for experiments, we use a simple Newton optimization method~\cite{poly_optimization_prims}).

Our focus is on reducing path planning search effort. Similar goals are addressed by approaches like Jump Point Search~\cite{jps}, which significantly accelerates search on an 8-connected grid, and the well-known WA*~\cite{wastar}, which provides suboptimal solutions more quickly using weighted heuristics.

\section{Problem Statement}
Consider a point-sized mobile agent moving in a 2D workspace $W \subset \mathds{R}^2$ composed of a free space $W_{free}$ and obstacles $W_{obs}$. The workspace is tessellated into a grid, where each cell $(i,j)$ is either free or blocked.

\textbf{State Representation}. The state of the agent is defined by a 3D vector $(x, y, \phi)$, with coordinates $(x, y) \in W$ and heading angle $\phi \in [0, 360^\circ)$. We assume that the latter can be discretized into a set $\Theta = \{\phi_1, ..., \phi_k\}$, allowing us to focus on discrete states $s = (i, j, \theta)$ that correspond to the centers of the grid cells $(i,j) \in \mathds{Z}^2$, with $\theta \in \Theta$.

\textbf{Motion Primitives}. The kinematic constraints and physical capabilities of the mobile agent are encapsulated in \textit{motion primitives}, each representing a short kinodynamically-feasible motion (i.e., a continuous state change). We assume these primitives align with the discretization, meaning that transitions occur between two discrete states, such as $(i, j, \theta)$ and $(i', j', \theta')$.
In other words, each motion starts and ends at the center of a grid cell, with its endpoint headings belonging to the finite set $\Theta$. Each primitive is additionally associated with the \emph{collision trace}, which is a sequence of cells swept by the agent when executing the motion, and \emph{cost} which is a positive number (e.g. the length of the primitive).

For a given state $s=(i, j, \theta)$ there is a finite number of motion primitives that the agent can use to move to other states. Moreover, we assume that there can be no more than one primitive leading to each other state, which corresponds to the intuitive understanding of a primitive as an elementary motion. We also consider that the space of discrete states, along with the primitives, is \textit{regular}; that is, for any $\delta_i, \delta_j \in \mathds{Z}$ if there exists a motion primitive connecting $(i, j, \theta)$ and $(i',j',\theta')$ then there also exists one from $(i+\delta_i,j+\delta_j,\theta)$ to $(i'+\delta_i,j'+\delta_j,\theta')$. While the endpoints of such primitives are distinct, the motion itself is not. This means that the collision traces of these primitives differ only by a parallel shift of the cells, and their costs coincide.
Thus, we can consider a canonical set of primitives, \textit{control set}, from which all others can be obtained through parallel translation. We assume that such a set is finite and is computed in advance according to the specific motion model of the mobile agent. To avoid ambiguity, we clarify the usage of the term ``primitive''. Unless specified otherwise, it refers to a specific motion between two discrete states. When we intend to refer to a motion template, we will explicitly state that the primitive is \textit{from the control set}. Such a template, instantiated at a discrete state, yields an exact primitive.

\textbf{Path}. A \textit{path} is a sequence of motion primitives, where the adjacent ones share the same discrete state. Its \textit{collision trace} is the union of the collision traces of the constituent primitives (shown as blue cells in Fig.~\ref{example_traj}). A path is \textit{collision-free} if its collision trace consists of free cells only.
    
\textbf{Problem}. The task is to find a collision-free path from a given start state $s_0$ to a goal state $s_f$. We wish to solve this problem optimally, i.e. to obtain the least cost path, where the cost of the path is the sum of the costs of its primitives.

\section{Method}

A well-established approach to solve the given problem is to search for a path on a \textit{state lattice} graph, where the vertices represent the discrete states and edges -- the primitives connecting them. In particular, heuristic search algorithms of the A* family can be used for such pathfinding.

These algorithms iteratively construct a search tree composed of the partial paths (sequences of the motion primitives). At each iteration, the most prominent partial path is chosen for extension. Extension is done by expanding the path's endpoint -- a discrete state $(i, j, \theta)$. This expansion involves considering all the primitives that can be applied to the state, checking which ones are valid (i.e. do not collide with the obstacles), computing the transition costs, filtering out the duplicates (i.e. the motions that lead to the states for which there already exist paths in the search tree at a lower or equal cost), and adding the new states to the tree.
Indeed, as the number of available motion primitives increases, the expansion procedure (which is the main building block of a search algorithm) becomes computationally burdensome, and the performance of the algorithm degrades.

Partially, this problem can be addressed by the \emph{lazy} approach, where certain computations associated with the expansion are postponed, most often -- collision checking. However, in environments with complex obstacle arrangements, searches that exploit lazy collision checking often extract invalid states (for which the collision check fails). Consequently, a significant amount of time is wasted on extra operations with the search tree, which is also time-consuming.

We propose an alternative approach. Instead of searching at the level of primitives, we search at the level of individual cells. During the search, we simultaneously reason about the motion primitives that can pass through the cells. This is achieved by defining a new search element that combines a cell with a set of primitives and a proper successor relationship. This approach allows for obtaining optimal solutions faster by leveraging its cell-by-cell nature.

\begin{figure}[t]
\centering
\includegraphics[width=0.46\textwidth]{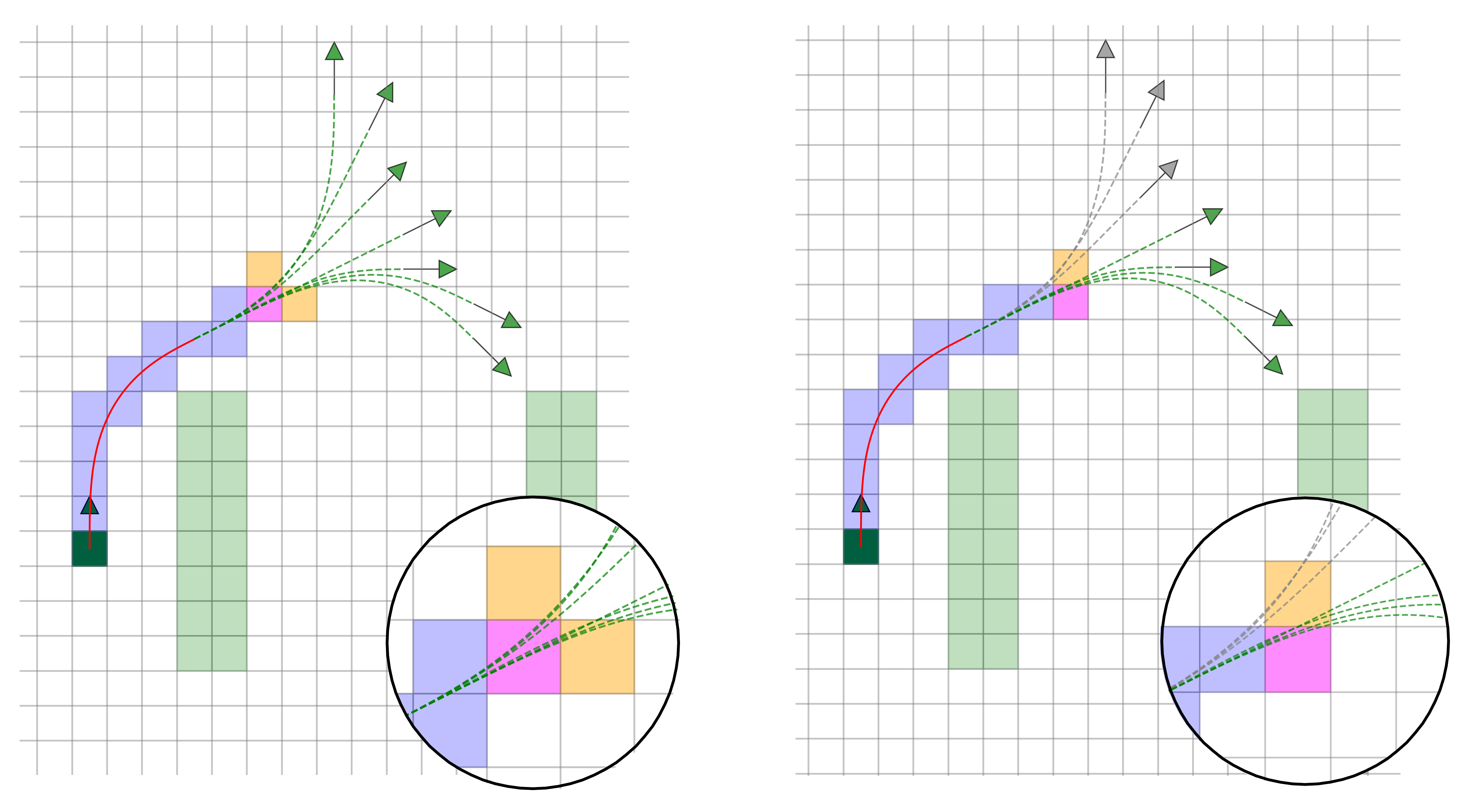}
\caption{Definition of cell successors during the construction of a collision trace.}
\label{example_mesh_graph}
\end{figure}

Fig.~\ref{example_mesh_graph} illustrates the intuition behind our search. On the left, a partial collision trace ends at a magenta cell under expansion. Assume this cell is reached by a red primitive, followed by any of the green ones, as their collision traces do not differ before this point. We store information about all primitives passing through a cell, forming what we call an \emph{extended cell}. This augmentation allows us to infer a small set of successors: the orange cells along the paths of the green primitives. For these successors, we generate corresponding extended cells, propagating the information about the primitives that pass through them. Later, when the search expands one such successor (e.g., the cell to the right, see the right panel of Fig.~\ref{example_mesh_graph}), its propagated information (gray primitives are excluded, as they led to a different cell) determines the cell above as the only valid continuation.


\subsection{The Search Space}
To define the elements of our search space, first, consider a set of $n$ arbitrary motion primitives from the control set, $prim_1$, ..., $prim_n$, and an index $k\in\mathbb{N}$ such that $\forall i\in \{1,\ldots,n\}:~k<U_{prim_i}$, where $U_{prim_i}$ is the number of cells in the collision trace of $prim_i$. Now, the following set of pairs is called the \textbf{configuration of primitives}:
\begin{equation*}
    \Psi=\left\{(prim_1,k), (prim_2,k), ..., (prim_n,k)\right\},
\end{equation*}
and used to define elements of the search space.

\begin{definition}[Extended cell]
An element of our search space is an \textbf{extended cell}, formally defined as a tuple of a specific grid cell $(i,j)$ and a configuration of primitives $\Psi$:
\begin{equation*}
    u = (i, j, \Psi)
\end{equation*}
\end{definition}

Recall that primitives from the control set act as motion templates, which, when instantiated at specific discrete states, yield specific primitives. Thus, the conceptual meaning of an extended cell is as follows: for each pair $(prim, k) \in \Psi$, we consider a copy of $prim$ traversing cell $(i, j)$ such that this cell is the $k$-th in its collision trace. Thus, an extended cell captures information both about the grid cell itself and about the motion primitives passing through it. The magenta cells in Fig.~\ref{example_mesh_graph} with green primitives passing through them illustrate this concept. The \textit{projection} of the extended cell $u = (i,j,\Psi)$ is the grid cell $(i, j)$.

In practice, instead of storing the configuration of primitives directly, a single number, assigned by indexing all possible configurations, can be used (see Appendix for details).

\begin{definition}[Initial Configuration]
For a given heading $\theta$, the initial configuration $\Psi_{\theta}$ is defined as:
\begin{equation*}
    \Psi_{\theta} = \left\{\left(prim_1, 1\right), \ldots, \left(prim_r, 1\right)\right\}, 
\end{equation*}
where $prim_1, \ldots, prim_r$ are all primitives from the control set with initial heading $\theta$ (see details in Algorithm~\ref{algorithm_initial_conf}).
\end{definition}
An extended cell containing such a configuration is an \textit{initial extended cell}. The latter can be viewed as a direct analog to a discrete state.

\begin{algorithm}[t]
\caption{Building the Initial Configuration}
\label{algorithm_initial_conf}
\textbf{Input:} Discrete angle $\theta \in \Theta$\\
\textbf{Function name:} \textsc{InitConf}($\theta$)
\begin{algorithmic}[1]
\State $\Psi \gets \emptyset$
\ForAll{$\mathit{prim} \in \texttt{ControlSet}$}
    \If{$\mathit{prim}$ emerges in $\theta$}
        \State $\Psi.\texttt{add}\left\{(\mathit{prim}, 1)\right\}$
    \EndIf
\EndFor
\State \Return $\Psi$
\end{algorithmic}
\end{algorithm}

\subsection{Successors}

To define the successor relationship between the extended cells, let us first denote the \textit{displacement} by
\begin{equation*}
\Delta^{prim}_k = (i' - i, j' - j)
\end{equation*}
\noindent This is the change in grid coordinates when transitioning from the $k$-th cell to the $(k+1)$-th cell in the collision trace of $prim$. Throughout the paper, we call the process of making such a transition a \textit{step along the primitive}. Thus, each step along the primitive is characterized by its displacement.

 Intuitively, the successors of an extended cell are obtained simply as the results of steps along each primitive from the current configuration (see Fig.~\ref{example_mesh_graph}). The formal definition is divided into two parts, depending on the type of successor.

\begin{definition}[Initial Successor]
\label{def:initial_successor}
Let $u=(i,j,\Psi)$ be an extended cell. Initial extended cell $v=(i',j',\Psi_{\theta})$ is called an \textbf{initial successor} of $u$ if the following condition holds:
\begin{align*}
    &\exists (prim, k) \in \Psi~~\text{such that:}~~k = U_{prim} - 1,\\
    &\Delta^{prim}_k = (i' - i, j' - j) \text{ and } prim \text{ ends at angle } \theta
\end{align*}
\end{definition}
\noindent
This condition ensures the existence of $prim$ in $\Psi$ that completes at $(i', j')$ with heading $\theta$. At this point, it can be extended by a whole bundle of primitives from $\Psi_{\theta}$, all starting with the same angle $\theta$.

\begin{definition}[Regular Successor]
\label{def:other_successor}
Let $v=(i',j',\Psi')$ and $u=(i,j,\Psi)$ be extended cells, with $v$ not being initial. Then $v$ is a \textbf{successor} of $u$ if the following hold simultaneously:
\begin{enumerate}
    \item \textbf{Predecessor Completeness:} For all primitives in $\Psi'$, their preceding cells must exist in the predecessor.
        \begin{align*}
                &\forall (prim, k) \in \Psi'~~\text{the following holds:}\\
                &(prim, k-1) \in \Psi~~\text{and}~~(i'-i, j'-j) = \Delta_{k-1}^{prim}
        \end{align*}

    \item \textbf{Successor Coverage:} All valid continuations from $\Psi$ to the successor's projection must be included in $\Psi'$.
    \begin{align*}
        &\forall (prim, k) \in \Psi~~\text{such that}~~k < U_{prim}-1~~\text{and}\\
        &\Delta_{k}^{prim}=(i'-i,j'-j)~~\text{the following holds:}\\
        &(prim, k+1) \in \Psi'
    \end{align*}
\end{enumerate}
\end{definition}
\noindent
Condition (2) requires that all primitives of the predecessor leading to the projection cell of the successor are present in its configuration, while Condition (1) ensures that there are no other primitives in the successor.

\subsection{Transition Costs}

To ensure compatibility with standard lattice-based methods, transition cost between extended cells is derived from the costs of the underlying motion primitives. Let $c_{prim}$ denote the cost of a primitive $prim$ from the control set.

\begin{definition}[Transition Cost]
\label{def:cost}
Let $v$ be a successor of an extended cell $u$. The transition cost from $u$ to $v$ is defined as:
\begin{align*}
    cost(u,v) = 
     \begin{cases}
        c_{prim}, & \text{if } v \text{ is an initial successor } \\
        0, & \text{otherwise }
    \end{cases}
\end{align*}
In the first case, $prim$ is a primitive terminating at $v$ that satisfies the conditions of Definition~\ref{def:initial_successor}. If multiple such primitives exist, any of their costs can be chosen. We will show later that this ambiguity cannot arise in any relevant case. 
\end{definition}
Algorithm~\ref{algorithm_extended_cell_successors} details the successor generation for a given extended cell $u = (i, j, \Psi)$. It uses a dictionary $\texttt{Confs}$ (line 2) to temporarily store the forming configurations of non-initial successors, mapped by their displacement from $u$. A set $\texttt{Successors}$ is initialized (line 3) to hold the resulting pairs of all successors and their transition costs.

\begin{algorithm}[ht]
\caption{Generating Successors of an Extended Cell}
\label{algorithm_extended_cell_successors}
\textbf{Input:} Extended cell $u$\\
\textbf{Output:} Set of pairs of successors and transition costs\\
\textbf{Function name:} \textsc{GetSuccessors}($u$)
\begin{algorithmic}[1]
\State $i,j,\Psi \gets u$
\State $\texttt{Confs} \gets \{\}$  \Comment{Dictionary}
\State $\texttt{Successors} \gets \emptyset$ \Comment{Set of successors and costs}
\ForAll{$(prim,k) \in \Psi$}  
    \State $(a,b) \gets \Delta_k^{prim}$ 
    \If{$k = U_{prim}-1$} \Comment{Case 1}
        \State $\theta \gets$ angle at which $prim$ ends
        \State $\Psi_1 \gets \textsc{InitConf}(\theta)$
        \State $v_1 \gets (i+a, j+b, \Psi_1)$
        \State $\texttt{Successors}.\texttt{add}\{(v_1, c_{prim})\}$ 
    \ElsIf{$(a,b) \notin \texttt{Confs}$} \Comment{Case 2}
        \State $conf_{new} \gets \{(prim,k+1)\}$
        \State $\texttt{Confs}[(a,b)] \gets conf_{new}$
    \Else
        \State $\texttt{Confs}[(a,b)].\texttt{add}\{(prim,k+1)\}$ 
    \EndIf
\EndFor       
\ForAll{$(a,b) \in \texttt{Confs}$} 
    \State $\Psi_2 \gets \texttt{Confs}[(a,b)]$
    \State $v_2 \gets (i+a,j+b,\Psi_2)$
    \State $\texttt{Successors}.\texttt{add}\{(v_2, 0)\}$
\EndFor
\State \Return $\texttt{Successors}$
\end{algorithmic}
\end{algorithm}

The main loop (lines 4-15) iterates over each pair $(prim, k)$ in the configuration $\Psi$. For each pair, it computes the displacement $(a,b)$ to the next cell in the primitive's collision trace (line 5). This displacement determines the grid coordinates $(i+a,j+b)$ of a potential successor cell relative to $u$. Based on whether the primitive terminates (checked in line 6), one of two cases is handled.

Case 1: Initial Successor (lines 6-10). If a primitive is at its final step ($k = U_{prim}-1$), an initial successor is generated and immediately stored in $\texttt{Successors}$ (line 10) according to Definition~\ref{def:initial_successor}. The angle $\theta$ at which the primitive ends determines the initial configuration $\Psi_1 = \textsc{InitConf}(\theta)$ for this successor (line 8). The transition cost is set to $c_{prim}$, the cost of the primitive itself, as this motion is now complete (see Definition~\ref{def:cost}).

Case 2: Regular Successor (lines 11-15). If a primitive does not terminate ($k < U_{prim}-1$), it contributes to a non-initial, or regular, successor. The displacement $(a,b)$ is used as a key in the $\texttt{Confs}$ dictionary. If this displacement is encountered for the first time (line 11), a new entry is created, and its configuration is initialized with the current primitive at its next step, $(prim, k+1)$ (lines 12-13). If the key $(a,b)$ already exists (line 14), it signifies that multiple primitives from $\Psi$ lead to the same grid cell $(i+a, j+b)$. In this case, $(prim, k+1)$ is added to the existing configuration for this displacement (line 15). Thus, for each displacement $(a,b)$, the $\texttt{Confs}$ dictionary groups all non-terminating primitives from $\Psi$ whose next step corresponds to this displacement. This process forms the complete configuration for each potential non-initial successor and precisely matches the conditions of Definition~\ref{def:other_successor}.

Finally (lines 16-19), the algorithm iterates through the $\texttt{Confs}$ dictionary. For each displacement $(a,b)$ and its aggregated configuration $\Psi_2$, it forms a single non-initial successor $v_2 = (i+a, j+b, \Psi_2)$. The transition cost to such a successor is $0$, per Definition~\ref{def:cost}. Thus, the pair $(v_2, 0)$ is added to the $\texttt{Successors}$ set (line 19).

\subsection{MeshA*}

Having defined the elements of the search space as well as the successor relationship, we obtain a directed weighted graph. We term this the \textit{mesh graph} to distinguish it from both the environment's grid representation and the standard lattice graph of motion primitives. This graph structure proves essential for our subsequent theoretical analysis. 

We can utilize a standard heuristic search algorithm, i.e. A*, to search for a path on this graph. We will refer to this approach as \textit{MeshA*}. Next we will show that running MeshA* leads to finding the optimal solution, which is equivalent to the one found by A* on the lattice graph.

\begin{algorithm}[t]
\caption{Trajectory Reconstruction}
\label{algorithm_reconstruction}
\textbf{Input:} Path $u_1, u_2, \ldots, u_N$ in the mesh graph between\\ initial extended cells\\
\textbf{Output:} Trajectory (chain of primitives)
\begin{algorithmic}[1]
\State $\texttt{Traj} \gets \emptyset$
\State $p \gets u_1$
\ForAll{$l = 2, 3, \ldots, N$}
    \If{$u_l$ is initial} \Comment{$u_l$ is next initial after $p$}
        \State $i_1, j_1, \Psi_{\theta_1} \gets p$
        \State $i_2, j_2, \Psi_{\theta_2} \gets u_l$
        \State $s_1 \gets (i_1, j_1, \theta_1)$  \Comment{Obtain discrete states}
        \State $s_2 \gets (i_2, j_2, \theta_2)$
        \State $prim \gets$ primitive from $s_1$ to $s_2$
        \State $\texttt{Traj}.\texttt{add}\{prim\}$
        \State $p \gets u_l$
    \EndIf
\EndFor       
\State \Return $\texttt{Traj}$ 
\end{algorithmic}
\end{algorithm}

\section{Theoretical Results}

We now establish the equivalence between searching on the mesh graph and searching on the state lattice, which infers that finding an optimal path on the state lattice is equivalent to finding an optimal path on the mesh graph followed by trajectory reconstruction. 

This section provides only proof sketches for brevity, while the detailed proofs can be found in the Appendix.


\begin{lemma}[Uniqueness of Path Cost]
For any path on the mesh graph starting from an initial extended cell, the cost of each transition is uniquely defined.
\label{lem:cost_uniqueness}
\end{lemma}
\begin{proof}[Proof Sketch]
The only non-zero costs occur on transitions into initial extended cells. We prove by contradiction that the primitive inducing such a transition is unique. The key insight is that any primitive can be traced backward along the mesh path to its origin. If two distinct primitives induced the same transition, this backward tracing procedure would show they connect the same discrete start and end states, which contradicts our problem statement.
\end{proof}

\begin{theorem}[From State Lattice to Mesh Graph]
Let there be a trajectory on the state lattice from a discrete state $s_a = (i_a, j_a, \theta_a)$ to $s_b = (i_b, j_b, \theta_b)$. Then there exists a path on the mesh graph from the initial extended cell $u_a = (i_a, j_a, \Psi_{\theta_a})$ to another initial one $u_b = (i_b, j_b, \Psi_{\theta_b})$ that satisfies the following conditions:
\begin{enumerate}
    \item The cost of this path is equal to the cost of the trajectory.
    \item The projections of the vertices along this path precisely form the
collision trace of this trajectory.
\end{enumerate}
\label{thm:lattice_to_mesh}
\end{theorem}

\begin{proof}[Proof Sketch]
Proof is constructive, by induction on the number of primitives. For the base case of a single primitive, the path is constructed by starting at $u_a$ and iteratively applying the successor definition to step along the primitive's collision trace cell by cell. Each step to a non-final cell of the trace generates a regular successor with zero cost. The final step generates an initial successor $u_b$ with cost $c_{prim}$, ensuring the total path cost matches the primitive. The inductive step shows that a path for a longer trajectory is simply the concatenation of the mesh paths for its constituent primitives, ensuring the total cost and trace are preserved.
\end{proof}

\begin{theorem}[From Mesh Graph to State Lattice]
For any path on the mesh graph from an initial extended cell $u_a = (i_a, j_a, \Psi_{\theta_a})$ to another initial $u_b = (i_b, j_b, \Psi_{\theta_b})$, Algorithm~\ref{algorithm_reconstruction} reconstructs a trajectory
composed of primitives (corresponding to the path in the state lattice)
that transition from the discrete state $s_a = (i_a, j_a, \theta_a)$ to $s_b = (i_b, j_b, \theta_b)$ and satisfy the following conditions: 
\begin{enumerate}
    \item The cost of this trajectory is equal to the cost of the path in the mesh graph. 
    \item The collision trace of this trajectory coincides with the projections
of the vertices along this mesh graph path.
\end{enumerate}
\label{thm:mesh_to_lattice}
\end{theorem}

\begin{proof}[Proof Sketch]
The proof is constructive, formalizing the reconstruction algorithm~\ref{algorithm_reconstruction}. The core idea is that any mesh path between initial cells is uniquely decomposed into segments, each connecting two consecutive initial cells in the path. We establish that each segment corresponds to a single, unique motion primitive. Concatenating these primitives reconstructs the full trajectory, and the cost is preserved as it is only incurred at the end of each segment.
\end{proof}

\begin{theorem}[Equivalence of Optimal Pathfinding]
The search for a least-cost, collision-free trajectory between discrete states $s_a=(i_a, j_a, \theta_a)$ and $s_b=(i_b, j_b, \theta_b)$ is equivalent to performing two steps:
\begin{enumerate}
    \item Find a least-cost path on the mesh graph between the corresponding initial extended cells $u_a = (i_a, j_a, \Psi_{\theta_a})$ and $u_b = (i_b, j_b, \Psi_{\theta_b})$. This path must be collision-free, meaning the projection of every vertex on the path is a free grid cell.
    \item Recover the trajectory from this path using Algorithm~\ref{algorithm_reconstruction}.
\end{enumerate}
The resulting trajectory will be optimal and collision-free.
\label{thm:main}
\end{theorem}

\begin{proof}[Proof Sketch]
The proof rests on the bidirectional, cost-preserving correspondence from Theorems \ref{thm:lattice_to_mesh} and \ref{thm:mesh_to_lattice}. Theorem \ref{thm:lattice_to_mesh} guarantees that any optimal collision-free trajectory can be mapped to a mesh path of identical cost. Since the path's projections match the trajectory's trace, this path is also collision-free. Conversely, Theorem \ref{thm:mesh_to_lattice} ensures that any optimal collision-free mesh path can be reconstructed into a collision-free trajectory of the same cost. Since a path exists in one space if and only if a path of the same cost exists in the other, their optimal costs must be identical, making the proposed two-step procedure both correct and complete.
\end{proof}

\begin{figure*}[th!]
\centering
\includegraphics[width=0.98\textwidth]{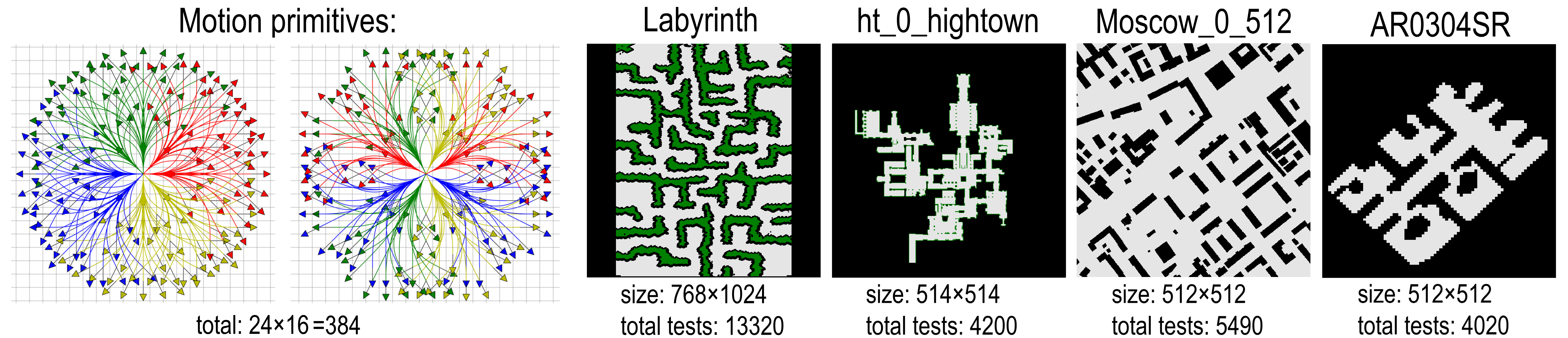}
\caption{Experimental setup: control set and MovingAI maps.}
\label{setup}
\end{figure*}

\section{On The Efficiency Of MeshA*}

MeshA* is not a new search algorithm, but the standard A* applied to a novel search space, the mesh graph (hence, we omit its pseudocode). Theorem~\ref{thm:main} established that A* on the mesh graph (MeshA*) is equivalent to A* on the state lattice (LBA*). We now show how the inherent structure of the mesh graph enables cell-level search-space pruning techniques unavailable to LBA*. We believe this structural advantage is the main reason MeshA* notably outperforms LBA*, as shown later in the Empirical Evaluation section.

We first define the heuristic function for MeshA*. To this end, we introduce a new concept related to an extended cell.

\begin{definition}[Primitive Endpoints]
For a non-initial extended cell $u=(i, j, \Psi)$, the set $\texttt{Finals}(u)$ is defined as:
\begin{align*}
    \texttt{Finals}(u) = \{ (u_{prim}, c_{prim}) \mid \exists (prim, k) \in \Psi \}
\end{align*}
where $u_{prim}$ is the initial extended cell where the instance of $prim$ passing through $u$ terminates.
\label{def:finals}
\end{definition}

Intuitively, $\texttt{Finals}(u)$ captures the information about completing each of the primitives passing through $u$. This is used to define the heuristic function for MeshA*.

\begin{definition}[Heuristic for MeshA*]
Let $h_{LBA*}$ be an (admissible and consistent) heuristic function for lattice-based A* (LBA*). Then, for any extended cell $u = (i,j,\Psi)$, the heuristic function for the MeshA* is:
\begin{align*}
    h(u) = 
     \begin{cases}
        h_{LBA*}(i, j, \theta),~~~~~~\text{if } u \text{ is initial with } \Psi = \Psi_{\theta}\\
        \min_{(u_0, c_0) \in \text{Finals}(u)} \{h(u_0) + c_0\},~~~\text{otherwise}
    \end{cases}
\end{align*}
\end{definition}

This definition preserves the admissibility and consistency of $h_{LBA*}$. For initial cells, it matches $h_{LBA*}$ directly. For non-initial cells, since any path from $u$ must continue through one of its primitive endpoints $u_0$ (at a cost of $c_0$ to complete), taking the minimum $\min(c_0 + h(u_0))$ over all such options enforces consistency by definition.

On the other hand, such a heuristic leads to a more extensive pruning of the search space for MeshA* for the following reason. For any non-initial cell $u$ $h(u)$ estimates the cost of the best possible trajectory that can be completed from $u$. Crucially, this allows us to evaluate the promise of each potential primitive endpoint in $\texttt{Finals}(u)$ independently. If one primitive leads towards a region with a high heuristic cost, the search will naturally deprioritize such paths, effectively abandoning that branch of the search long before the primitive is fully traversed. In other words, unlike LBA* the introduced method, MeshA*, may detect the unperspective search branches much earlier, thanks to the cell-by-cell nature of the search.

Next, the structure of the mesh graph also enables a powerful terminal pruning rule. A non-initial extended cell $u$ can be safely pruned (i.e., its expansion skipped) if all of its endpoint cells in $\texttt{Finals}(u)$ have already been expanded. Indeed, since any path from $u$ must pass through one of these endpoints, and our search strategy with a consistent heuristic guarantees optimality (or bounded suboptimality in the weighted case) without re-openings \cite{wo_reexp}, further exploration from $u$ is redundant.

\section{Empirical Evaluation}

\textbf{Setup.} In the experiments we utilize $16$ different headings and generate $24$ motion primitives for each heading using the car-like motion model. Experiments are conducted on four MovingAI benchmark~\cite{movingai} maps of varying topology. Start-goal pairs are taken from benchmark scenarios, with three randomly generated headings per pair, yielding over 25,000 instances in total. See Figure~\ref{setup}.

The following algorithms were evaluated:
\begin{enumerate}
    \item \textbf{LBA*} (short for Lattice-based A*): The standard A* algorithm on the state lattice, serving as the baseline.
    \item \textbf{LazyLBA*}: The same algorithm that conducts collision-checking lazily.
    \item \textbf{MeshA*} (ours): Running A* on the mesh graph.
\end{enumerate}

The cost of each primitive is its length, and the heuristic is the Euclidean distance. We test heuristic weights $w \in \{1, 1.1, 2, 5, 10\}$, where larger weights speed up search but increase solution cost~\cite{wastar}.

To ensure a fair comparison, we implement all algorithms in C++ with identical data structures (e.g., priority queue) and the same underlying A* logic for both the mesh graph (MeshA*) and the state lattice (LBA*). Both implementations include \textit{g-value pruning} (i.e., avoiding exploration a state that already has a known better g-value). To avoid \texttt{DecreaseKey} operations in the \texttt{PriorityQueue} we perform this "lazily": all successors are added to \texttt{OPEN}, but a search node is discarded upon extraction if its state is already in \texttt{CLOSED}. This strategy correctly guarantees optimality (for $w=1$) and bounded suboptimality (for $w>1$) without re-opening \cite{wo_reexp}.

All experiments were conducted on AMD EPYC 7742 under identical conditions. We primarily measure runtime, as other metrics like nodes generated and expansions are not directly comparable due to fundamentally different search approaches: LBA* expands motion primitives while MeshA* expands extended cells. Additionally, we report \textit{checked cells} (total cells in collision traces examined) as a processor-independent metric that highlights MeshA*'s cell-by-cell search advantage.

\begin{figure}[t]
\centering
\includegraphics[width=0.47\textwidth]{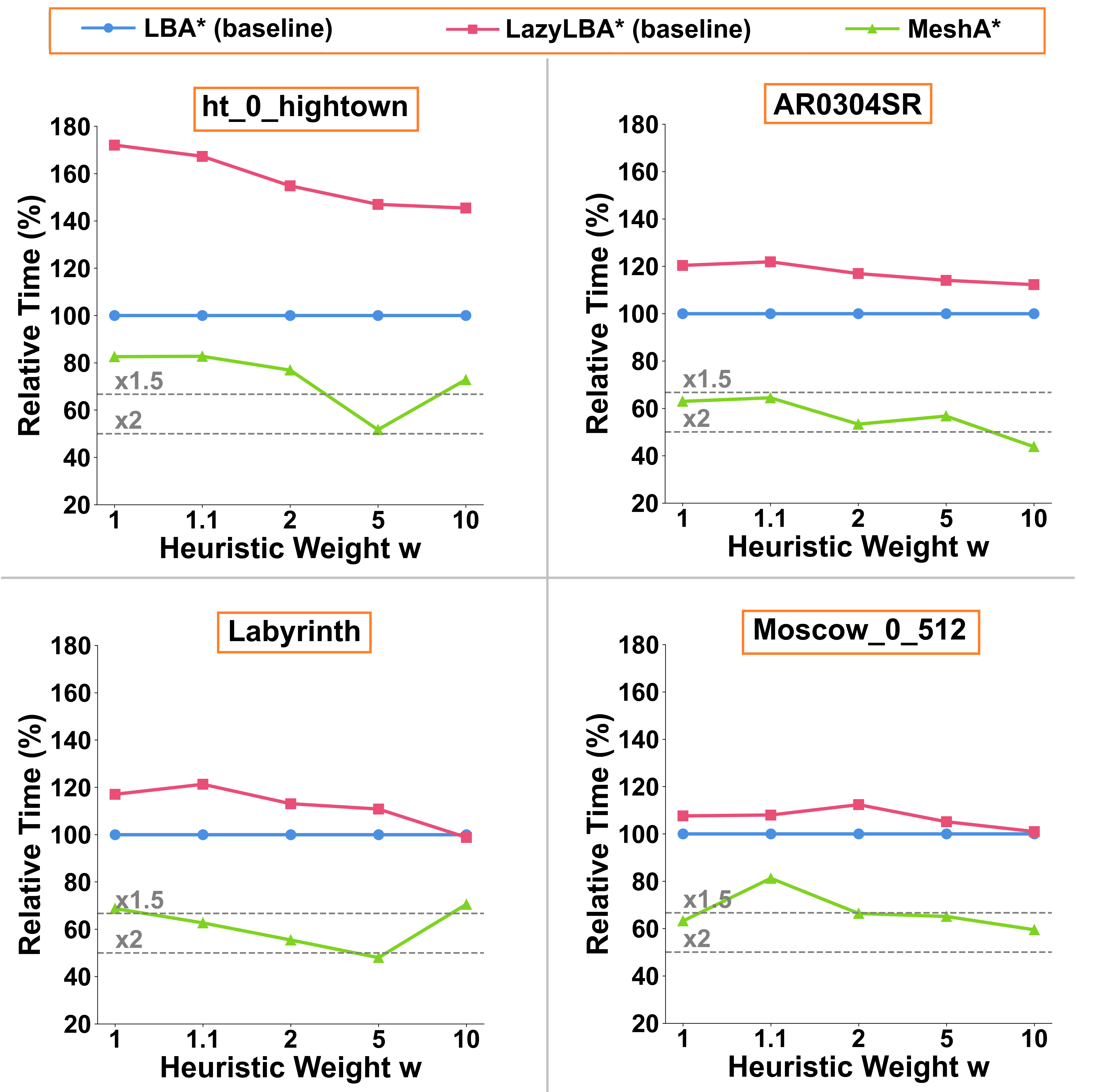}
\caption{Median runtime of the evaluated algorithms (relative to the runtime of LBA*). The lower -- the better. 50\% corresponds to x2 speed-up.}
\label{exp_time}
\end{figure}

\noindent \textbf{Results.} Fig.~\ref{exp_time} illustrates the median runtime of each algorithm as a ratio to LBA*. That is, the runtime of LBA* is considered 100\% in each run, and the runtime of the other solvers is divided by this value. Thus, the lower the line is on the plot -- the better. Clearly, MeshA* consistently outperforms the baselines. For $w=1$, it consumes on average $60\%$-$80\%$ of the LBA* runtime (depending on the map). That is, the speed-up is up to 1.5x. When the heuristic is weighted, this gap gets even more pronounced, and we can observe x2 speed-up.

Interestingly, LazyLBA* consistently performs worse than LBA*. This can be attributed to the simplicity of collision checking in our experiments, which involves verifying that the cells in the collision trace of each primitive are not blocked. This is a quick check in our setup, and postponing it does not make sense. Moreover, lazy collision checking begins to incur additional time costs from numerous pushing and popping of unnecessary vertices to the search queue, which standard LBA* prunes through collision checking. This effect is especially visible on the \texttt{ht\_0\_hightown} map, which contains numerous narrow corridors and passages where many primitives lead to collisions.

Fig.~\ref{exp_checked} demonstrates another crucial advantage of MeshA*. Remarkably, MeshA* processes only $\approx 50\%$ of the cells that LazyLBA* examines. We compare against LazyLBA* as it processes fewer cells than standard LBA* by deferring collision checks until primitive expansion. This significant reduction highlights MeshA*'s cell-by-cell search advantage: it can terminate unpromising primitives early, while LBA* must process entire primitives even when only their initial segments are relevant (see prev. section).

The costs of the obtained solutions are summarized in Table~\ref{tab:algorithm_costs} for the \texttt{ht\_0\_hightown} map (results for other maps are similar and provided in the Appendix). Each cell shows the median cost relative to the optimal one, thus lower values are better ($100\%$ indicates optimal solutions). As expected, LBA*, LazyLBA*, and MeshA* are optimal.

\begin{figure}[t]
\centering
\includegraphics[width=0.47\textwidth]{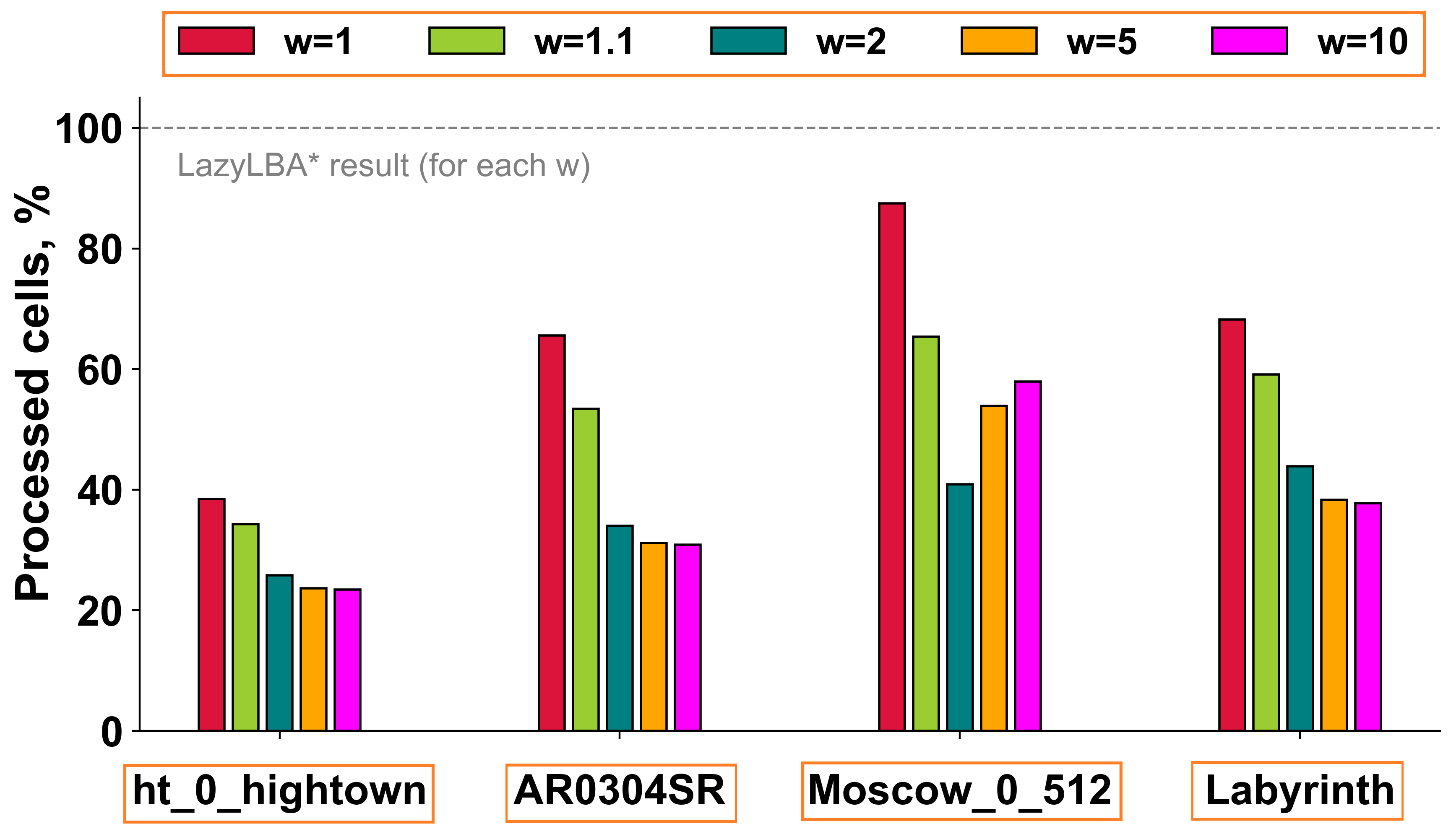}
\caption{Median number of processed grid cells (sum over all collision traces) for MeshA* relative to LazyLBA*}
\label{exp_checked}
\end{figure}

Another observation is that MeshA*'s solution costs increase slightly more with rising $w$ than LBA*'s. This is because LBA* computes the heuristic only twice per primitive (start/end), while MeshA* computes it at every cell. Thus, the weight $w$ has a more pronounced effect on MeshA*. This same property, however, explains MeshA*'s improved performance with weighted heuristics: the more frequent evaluations allow for more aggressive pruning of unpromising directions, achieving the demonstrated 2x speedup.

Overall, our experiments confirm that the suggested search approach, MeshA*, consistently outperforms the baselines and is much faster in practice (1.5x faster when searching for optimal solutions and 2x faster when searching for the suboptimal ones).

\begin{table}[t]
    \centering
    \begin{tabular}{c|cccc}
        \toprule
        \multirow{2}{*}{Algorithms}& \multicolumn{4}{c}{\texttt{ht\_0\_hightown}}\\ 
        \cmidrule(lr){2-5} 
        & $w=1$ & $w=2$ & $w=5$ & $w=10$ \\
        \midrule
        LBA* & {$100.0$} & {$105.7$} & {$110.0$} & {$113.2$} \\
                LazyLBA* & {$100.0$} & {$105.7$} & {$110.0$} & {$113.2$} \\
                MeshA* & {$100.0$} & {$109.2$} & {$117.6$} & {$122.3$} \\
        \bottomrule

    \end{tabular}
    \caption{Median relative costs of the trajectories found as a percentage of the optimal cost.}
    \label{tab:algorithm_costs}
\end{table}

\section{Conclusion}
In this paper, we have considered a problem of finding a path composed of motion primitives that are aligned with the grid representation of the workspace. We have suggested a novel way to systematically search for a solution by reasoning over the sequences of augmented grid cells rather than over sequences of motion primitives. The resultant solver, MeshA*, is provably complete and optimal and is notably faster than the regular A* search on motion primitives.

Despite having considered path planning in 2D when the agent's state is defined as $(x, y, \theta)$, the idea that stands behind MeshA* is applicable to pathfinding in 3D as well as to the cases where the agent's state contains additional variables as long as they can be discretized. Adapting MeshA* to such setups is a prominent direction for future work.

\bibliography{aaai2026}

@ARTICLE{review,
  author={González, David and Pérez, Joshué and Milanés, Vicente and Nashashibi, Fawzi},
  journal={IEEE Transactions on Intelligent Transportation Systems}, 
  title={A Review of Motion Planning Techniques for Automated Vehicles}, 
  year={2016},
  volume={17},
  number={4},
  pages={1135-1145},
  doi={10.1109/TITS.2015.2498841}}

@article{sampling,
    author = {Sertac Karaman and Emilio Frazzoli},
    title ={Sampling-based algorithms for optimal motion planning},
    journal = {The International Journal of Robotics Research},
    volume = {30},
    number = {7},
    pages = {846-894},
    year = {2011},
    doi = {10.1177/0278364911406761}}

@article{pivtoraiko2009,
    author = {Pivtoraiko, Mihail and Knepper, Ross A. and Kelly, Alonzo},
    title = {Differentially constrained mobile robot motion planning in state lattices},
    journal = {Journal of Field Robotics},
    volume = {26},
    number = {3},
    pages = {308-333},
    doi = {https://doi.org/10.1002/rob.20285},
    year = {2009}
}

@Article{prims_in_sampling,
    author={Sakcak, Basak
    and Bascetta, Luca
    and Ferretti, Gianni
    and Prandini, Maria},
    title={Sampling-based optimal kinodynamic planning with motion primitives},
    journal={Autonomous Robots},
    year={2019},
    month={Oct},
    day={01},
    volume={43},
    number={7},
    pages={1715-1732},
    issn={1573-7527},
    doi={10.1007/s10514-019-09830-x}}

@INPROCEEDINGS{shoot_method_prims,
  author={Jeon, Jeong hwan and Karaman, Sertac and Frazzoli, Emilio},
  booktitle={2011 50th IEEE Conference on Decision and Control and European Control Conference}, 
  title={Anytime computation of time-optimal off-road vehicle maneuvers using the RRT*}, 
  year={2011},
  volume={},
  number={},
  pages={3276-3282},
  doi={10.1109/CDC.2011.6161521}}

@INPROCEEDINGS{b_spline_prims,
  author={Flores, M.E. and Milam, M.B.},
  booktitle={2006 American Control Conference}, 
  title={Trajectory generation for differentially flat systems via NURBS basis functions with obstacle avoidance}, 
  year={2006},
  volume={},
  number={},
  pages={7 pp.-},
  doi={10.1109/ACC.2006.1657645}}

@INPROCEEDINGS{learn_prims,
  author={De Iaco, Ryan and Smith, Stephen L. and Czarnecki, Krzysztof},
  booktitle={2019 IEEE Intelligent Vehicles Symposium (IV)}, 
  title={Learning a Lattice Planner Control Set for Autonomous Vehicles}, 
  year={2019},
  volume={},
  number={},
  pages={549-556},
  doi={10.1109/IVS.2019.8813797}}

@Article{cover_prims,
    author={Yakovlev, K. S.
    and Andreychuk, A. A.
    and Belinskaya, J. S.
    and Makarov, D. A.},
    title={Safe Interval Path Planning and Flatness-Based Control for Navigation of a Mobile Robot among Static and Dynamic Obstacles},
    journal={Automation and Remote Control},
    year={2022},
    month={Jun},
    day={01},
    volume={83},
    number={6},
    pages={903-918},
    issn={1608-3032},
    doi={10.1134/S000511792206008X}}

@article{poly_optimization_prims,
  title={Trajectory generation for car-like robots using cubic curvature polynomials},
  author={Nagy, Bryan and Kelly, Alonzo},
  journal={Field and Service Robots},
  year={2001}}

@ARTICLE{astar,
  author={Hart, Peter E. and Nilsson, Nils J. and Raphael, Bertram},
  journal={IEEE Transactions on Systems Science and Cybernetics}, 
  title={A Formal Basis for the Heuristic Determination of Minimum Cost Paths}, 
  year={1968},
  volume={4},
  number={2},
  pages={100-107},
  doi={10.1109/TSSC.1968.300136}}

@inproceedings{pivtoraiko2005,
  title={Efficient constrained path planning via search in state lattices},
  author={Pivtoraiko, Mihail and Kelly, Alonzo},
  booktitle={International Symposium on Artificial Intelligence, Robotics, and Automation in Space},
  pages={1--7},
  year={2005},
  organization={Munich Germany}}

@article{movingai,
  title={Benchmarks for Grid-Based Pathfinding},
  author={Sturtevant, N.},
  journal={Transactions on Computational Intelligence and AI in Games},
  volume={4},
  number={2},
  pages={144 -- 148},
  year={2012}}

@article{wastar,
title = {Weighted A* search – unifying view and application},
journal = {Artificial Intelligence},
volume = {173},
number = {14},
pages = {1310-1342},
year = {2009},
issn = {0004-3702},
doi = {https://doi.org/10.1016/j.artint.2009.06.004},
author = {Rüdiger Ebendt and Rolf Drechsler}}

@article{rrt,
author = {LaValle, Steven and Kuffner, James},
year = {2001},
month = {01},
pages = {378-400},
title = {Randomized Kinodynamic Planning.},
volume = {20},
journal = {I. J. Robotic Res.}}

@article{rrt_star,
  title={Incremental Sampling-based Algorithms for Optimal Motion Planning},
  author={Sertac Karaman and Emilio Frazzoli},
  journal={ArXiv},
  year={2010},
  volume={abs/1005.0416}}

@INPROCEEDINGS{informed_rrt,
  author={Gammell, Jonathan D. and Srinivasa, Siddhartha S. and Barfoot, Timothy D.},
  booktitle={2014 IEEE/RSJ International Conference on Intelligent Robots and Systems}, 
  title={Informed RRT*: Optimal sampling-based path planning focused via direct sampling of an admissible ellipsoidal heuristic}, 
  year={2014},
  volume={},
  number={},
  pages={2997-3004},
  doi={10.1109/IROS.2014.6942976}}

@article{jps,
  title={Online Graph Pruning for Pathfinding On Grid Maps},
  author={Daniel Damir Harabor and Alban Grastien},
  journal={Proceedings of the AAAI Conference on Artificial Intelligence},
  year={2011}}

@article{rrt_connect,
  title={RRT-connect: An efficient approach to single-query path planning},
  author={James J. Kuffner and Steven M. LaValle},
  journal={Proceedings 2000 ICRA. Millennium Conference. IEEE International Conference on Robotics and Automation. Symposia Proceedings (Cat. No.00CH37065)},
  year={2000},
  volume={2},
  pages={995-1001 vol.2}}

@inproceedings{rrt_x,
  title={RRTX: Real-Time Motion Planning/Replanning for Environments with Unpredictable Obstacles},
  author={Michael W. Otte and Emilio Frazzoli},
  booktitle={Workshop on the Algorithmic Foundations of Robotics},
  year={2014}}

@article{wo_reexp,
author = {Chen, Jingwei and Sturtevant, Nathan},
year = {2021},
month = {05},
pages = {3688-3696},
title = {Necessary and Sufficient Conditions for Avoiding Reopenings in Best First Suboptimal Search with General Bounding Functions},
volume = {35},
journal = {Proceedings of the AAAI Conference on Artificial Intelligence},
doi = {10.1609/aaai.v35i5.16485}
}

\section{Appendix A. Proofs of Theoretical Statements}

\subsection{Lemma 1}
\begin{proof}
Let a path \( u_1, u_2, \ldots, u_N \) on the mesh graph start at an initial extended cell $u_1$. We use the notation \( u_l := (i_l, j_l, \Psi_l) \) for each $u_l$ in the path.

Consider any transition $u_{l-1} \to u_l$ in this path. If $u_l$ is not an initial extended cell, its cost is uniquely defined as $0$. Suppose $u_l=(i_l, j_l, \Psi_{\theta_l})$ is initial. By Definition~\ref{def:cost}, the transition cost is $c_{prim}$ for a primitive $prim$ satisfying the conditions of an initial successor (see Definition~\ref{def:initial_successor}). We must prove this $prim$, and thus its cost, is unique.

Assume for contradiction that there are two distinct such primitives, $prim_A$ and $prim_B$. This means that for both $prim \in \{prim_A, prim_B\}$:
\begin{align*}
    &1.~~(prim, U_{prim}-1) \in \Psi_{l-1}\\
    &2.~~prim~\text{ terminates at angle }~\theta_l\\
    &3.~~\text{The displacement matches the step to }~u_k:\\
    &~~~~~~\Delta_{U_{prim}-1}^{prim} = (i_l - i_{l-1}, j_l - j_{l-1})
\end{align*}

 Let, for brevity, $k = U_{prim}-1$. We will now perform a procedure we call \textbf{tracing a primitive backward} along the mesh graph path. Since $(prim, k) \in \Psi_{l-1}$, by the definition of a successor, it follows that $(prim, k-1) \in \Psi_{l-2}$ and $\Delta_{k-1}^{prim} = (i_{l-1}-i_{l-2},~j_{l-1}-j_{l-2})$. Continuing in a similar manner, we obtain that $(prim, 1) \in \Psi_{l-k}$ and $\forall t \in \{l-k, \ldots, l-1\}$ it holds: $\Delta_{k-(l-1)+t}^{prim} = (i_{t+1} - i_t,~j_{t+1} - j_t)$. 

Note that the vertex $u_{l-k}$ exists. Indeed, since $l > 1$, among the previous vertices relative to $u_l$ on the path, there must be at least one initial vertex (at the very least, $u_1$). Therefore, by continuing in a similar manner, we will inevitably reach the extended cell where $prim$ originated, namely $u_{l-k}$. It is also important to note that the extended cell $u_{l-k}$ is initial, as it is either equal to $u_1$ (in which case the fact that $u_{l-k}$ is initial follows from the condition) or it is a successor of some $u_{l-k-1}$. In this case, $u_{l-k}$ will also be initial, as this is the only scenario according to the definition of a successor where a configuration can contain a pair of a primitive with $1$. Thus, since $u_{l-k}$ is initial, we have $\Psi_{l-k} = \Psi_{\phi}$ for some discrete angle $\phi$. 

Finally, we have obtained that $(prim, 1) \in \Psi_{\phi}$ and for all $t$ such that $l-k \leq t \leq l-1$, it holds that $\Delta_{k-(l-1)-t}^{prim} = (i_{t+1} - i_t,~j_{t+1} - j_t)$. Considering that $prim$ is a primitive of the control set (i.e., it is regarded as a motion template), the obtained result guarantees that if a copy of $prim$ is placed at $(i_{l-k}, j_{l-k})$, it will yield a primitive that transitions from the discrete state $s_1 = (i_{l-k}, j_{l-k}, \phi)$ to another discrete state $s_2 = (i_l, j_l, \theta)$. 

This reasoning can be applied both to $prim_A$ and $prim_B$. As a result, we have now established that both $prim_A$ and $prim_B$ connect the same discrete start state $s_1$ to the same discrete end state $s_2$. Since we assumed $prim_A$ and $prim_B$ are distinct, this contradicts our problem's core assumption that at most one primitive connects any two discrete states.
Therefore, the primitive inducing the transition to an initial extended cell is unique, making its cost $c_{prim}$ uniquely defined. The lemma holds.
\end{proof}

\subsection{Theorem 2}

\begin{proof}
The proof is by induction on the number of primitives, $N$, in the trajectory.

\textbf{Base Case ($N=1$):} The trajectory consists of a single primitive, $prim$, from $s_a$ to $s_b$. Its collision trace is a sequence of cells $(c_1, \ldots, c_m)$, where $c_1=(i_a, j_a)$ and $c_m=(i_b, j_b)$. By definition, $(prim, 1) \in \Psi_{\theta_a}$. We construct a mesh path starting with $u_1 = (c_1, \Psi_{\theta_a})$. For each step from cell $c_k$ to $c_{k+1}$ along the primitive, there is a displacement $\Delta_k^{prim} = c_{k+1} - c_k$. By Definition \ref{def:other_successor}, since $(prim,k) \in \Psi_k$ (the configuration of extended cell $u_k$), there exists a successor $u_{k+1}=(c_{k+1}, \Psi_{k+1})$ where $(prim, k+1) \in \Psi_{k+1}$. This constructive process yields a path $u_1, \ldots, u_m$ whose projections match the collision trace. The final transition $u_{m-1} \to u_m$ is to an initial cell (by Definition 2), as $prim$ completes. By Definition \ref{def:cost}, its cost is $c_{prim}$, while all prior transition costs are 0. The total path cost is thus $c_{prim}$, matching the trajectory's cost.

\textbf{Inductive Step:} Assume the theorem holds for any trajectory of $N-1$ primitives. A trajectory of $N$ primitives from $s_a$ to $s_b$ can be decomposed into a trajectory of $N-1$ primitives from $s_a$ to some intermediate state $s_{mid}=(i_m,j_m,\theta_m)$, followed by a final primitive from $s_{mid}$ to $s_b$.
By the induction hypothesis, a mesh path $P_{N-1}$ exists from $u_a$ to $u_{mid}=(i_m,j_m,\Psi_{\theta_m})$ with the same cost and trace. By the base case, a path $P_1$ exists from $u_{mid}$ to $u_b$ for the final primitive. Concatenating these at the shared initial cell $u_{mid}$ yields a single mesh path $P_N = P_{N-1} \circ P_1$ from $u_a$ to $u_b$ whose cost and projected trace match the full trajectory. The induction holds.
\end{proof}

\subsection{Theorem 3}
\begin{proof}
The proof is a constructive argument demonstrating the correctness of Algorithm~\ref{algorithm_reconstruction}. The core idea is to show that any mesh path between two initial cells can be uniquely decomposed into a sequence of segments, where each segment corresponds to exactly one motion primitive.

Let the mesh path be $P = (u_1, u_2, \ldots, u_N)$, where $u_1=u_a$ and $u_N=u_b$. Algorithm~\ref{algorithm_reconstruction} works by iterating through the sequence of initial extended cells on this path. Let these initial cells be $u_{p_1}, u_{p_2}, \ldots, u_{p_M}$, where $p_1=1$ and $p_M=N$. The algorithm considers segments of the path between any two consecutive initial cells, e.g., from $u_{p_j}$ to $u_{p_{j+1}}$.

For each such segment, we must show that it corresponds to a unique motion primitive. Let's consider the segment from $u_p = (i_p, j_p, \Psi_{\theta_p})$ to the next initial cell $u_l = (i_l, j_l, \Psi_{\theta_l})$. From the proof of Lemma~\ref{lem:cost_uniqueness}, the transition into $u_l$ must be induced by a unique primitive, let's call it $prim$. Now we must show that this $prim$ fully accounts for the entire mesh path segment from $u_p$ to $u_l$.

Using the \textbf{tracing a primitive backward} procedure, as detailed in the proof of Lemma~\ref{lem:cost_uniqueness}, we trace $prim$ back from $u_l$. This procedure follows the mesh path vertices backward, step-by-step. Since there are no other initial cells between $u_p$ and $u_l$, this tracing process must necessarily lead all the way back to $u_p$. This confirms that $prim$ originates from the initial state $s_p = (i_p, j_p, \theta_p)$ and terminates at $s_l = (i_l, j_l, \theta_l)$.

Furthermore, the sequence of displacements $\Delta$ during the backward trace of $prim$ must, by definition of the successor relationship, match the displacements between the corresponding vertices of the mesh path segment. This means the projection of the vertices of the mesh path segment $(u_p, \ldots, u_l)$ exactly matches the collision trace of this unique primitive $prim$.

Algorithm~\ref{algorithm_reconstruction} formalizes this decomposition. It iterates through the consecutive initial cells, identifies the unique primitive for each segment, and adds it to the trajectory. The total cost of the resulting trajectory is the sum of the costs of these identified primitives. The cost of each mesh path segment is precisely the cost of its corresponding primitive (as all other internal transitions have zero cost). Therefore, the total trajectory cost equals the total mesh path cost. The collision trace property also holds for the full trajectory by concatenating the traces of each primitive segment.

Thus, Algorithm~\ref{algorithm_reconstruction} correctly reconstructs a trajectory from the mesh graph path, satisfying both conditions from the statement of theorem.
\end{proof}

\subsection{Theorem 4 (Main Theorem)}
\begin{proof}
Let $c^*$ be the cost of an optimal (least-cost, collision-free) trajectory on the state lattice from $s_a$ to $s_b$. As this trajectory is collision-free, its entire collision trace consists of free cells. By Theorem~\ref{thm:lattice_to_mesh}, there exists a corresponding mesh path from $u_a$ to $u_b$ with the same cost $c^*$. The projections of this mesh path's vertices coincide with the trajectory's collision trace, and are therefore also composed entirely of free cells. This means the mesh path is collision-free by our definition. The existence of such a path implies that the cost of an optimal collision-free mesh path, $c'_{opt}$, can be no more than $c^*$, i.e., $c'_{opt} \le c^*$.

Conversely, let an optimal collision-free mesh path from $u_a$ to $u_b$ exist with cost $c'_{opt}$. By Theorem~\ref{thm:mesh_to_lattice}, Algorithm~\ref{algorithm_reconstruction} can construct a trajectory from this path with the same cost $c'_{opt}$. The collision trace of this trajectory will coincide with the mesh path's vertex projections. Since the mesh path is collision-free, its vertex projections are all free cells, meaning the trajectory is also collision-free. The existence of such a trajectory implies that the cost of an optimal state-lattice trajectory, $c^*$, can be no more than $c'_{opt}$, i.e., $c^* \le c'_{opt}$.

Combining the two inequalities, we conclude that the optimal costs are identical: $c'_{opt} = c^*$.

Therefore, the procedure outlined is correct and complete. Step 1, finding a least-cost, collision-free path on the mesh graph (e.g., using A*), will yield a path with the optimal cost $c^*$. Step 2, applying Algorithm~\ref{algorithm_reconstruction}, will then convert this optimal mesh path into a trajectory that is also optimal (as it has cost $c^*$) and collision-free. This establishes the equivalence.
\end{proof}

\section{Appendix B. Implementation Details}

The performance of any search-based planner is tied to the efficiency of its fundamental operations, primarily successor generation. Both LBA* and MeshA* operate on graphs derived from the same set of motion primitives, but their different state representations necessitate different implementation strategies for this core task. In this section, we detail the implementation of MeshA* that ensures its efficiency is comparable to LBA*.

\subsection{On-the-Fly Successor Generation}

To conserve memory, search-based planners typically generate their search graph on-the-fly, expanding only the nodes relevant to the current problem instance. This means that successors for any given node must be computed efficiently at the moment of its expansion.

For LBA*, a node in the search graph is a discrete state $s=(i, j, \theta)$. Successor generation from such a state is straightforward. The heading $\theta$ directly determines which primitives from the control set are applicable. For each applicable primitive, its geometric properties -- specifically, its displacement $(\delta_i, \delta_j)$ and its final heading $\theta'$ -- are pre-defined and known in advance. Consequently, computing a successor state is a simple and computationally inexpensive operation: for a primitive with displacement $(\delta_i, \delta_j)$ and final heading $\theta'$, the successor state is simply $(i+\delta_i, j+\delta_j, \theta')$.

For MeshA*, a node is an extended cell $u=(i, j, \Psi)$. A direct implementation of the successor definitions (Definitions~\ref{def:initial_successor}, \ref{def:other_successor}, and \ref{def:cost}) would require a runtime analysis of the configuration $\Psi$, which is computationally expensive.

To enable MeshA* to generate successors with an efficiency comparable to LBA*, we perform a one-time pre-computation stage. This stage involves two steps: first, numbering all reachable configurations to create a compact state representation, and second, building a precomputed transition table. This is conceptually analogous to the initial generation of the control set itself: both are one-time costs for a given motion model, performed before any planning tasks. 

Ultimately, LBA* and MeshA* can be seen as two different ways of structuring the same underlying search space defined by the control set. LBA* structures it as a state lattice, while MeshA* structures it as a mesh graph of extended cells. To search either structure efficiently, some form of pre-processing is required: for LBA*, this is the generation of the control set itself; for MeshA*, it is the analysis of this control set's connectivity, which we detail in the following sections.

\subsection{Numbering Configurations of Primitives}
It is important to note that the configuration of the primitives $\Psi$ is a complex object to utilize directly in the search process. However, it is evident that there is a finite number of such configurations (since the number of primitives in the control set is finite), allowing us to assign a unique number to each configuration and use only this number in the search instead of the complex set.

\begin{algorithm}[htbp]
\caption{Numbering Configurations of Primitives}
\label{algorithm_dfs}
\textbf{Output:} Dictionary $\texttt{Numbers}$ that maps each configuration of primitives to its corresponding number (i.e., its ID)\\
\textbf{Function name:} \textsc{MainProcedure}()
\begin{algorithmic}[1]
\State $n \gets 0$  \Comment{Initialize the variable for the number}
\State $\texttt{Numbers} \gets \{\}$
\ForAll{discrete heading $\theta$}
    \State $\Psi \gets \textsc{InitConf}(\theta)$
    \State $u \gets (0, 0, \Psi)$  \Comment{Next initial cell}
    \State \textsc{NumberingDFS}($u$) \Comment{Run DFS defined below}
\EndFor       
\State \Return $\texttt{Numbers}$
\end{algorithmic}
\vspace{0.3cm}
\textbf{Function name:} \textsc{NumberingDFS}($u$)
\begin{algorithmic}[1]
\State $i, j, \Psi \gets u$
\If{$\Psi \in \texttt{Numbers}$} 
     \State \Return
\EndIf
\State $\texttt{Numbers}[\Psi] \gets n$
\State $n \gets n+1$
\ForAll{$(v, cost) \in \textsc{GetSuccessors}(u)$}
    \State $\textsc{NumberingDFS}(v)$
    \State \textcolor{gray}{\small // Optional: Precompute transition table (see App. B)}
    \State \textcolor{gray}{\small // Let $\Psi$ and $v.\Psi$ have IDs $n_u, n_v$ from \texttt{Numbers}}
\State \textcolor{gray}{$\texttt{TransTable}[n_u].\text{add}\left\{(n_v,\ cost,\ v.i\!-\!i,\ v.j\!-\!j)\right\}$}
\EndFor
\end{algorithmic}
\end{algorithm}

According to Theorem \ref{thm:main}, we are only interested in paths originating from some initial extended cell. Therefore, it makes sense to number only the configurations that are reachable from these extended cells (which will be defined more formally below). To achieve this, we can fix a set of extended cells with all possible initial configurations of the primitives (the number of which corresponds to the number of discrete headings) and then perform a depth-first search from each of these cells. This search will assign a number to each newly encountered configuration while stopping at already numbered configurations. The pseudocode for the described idea is implemented in Algorithm \ref{algorithm_dfs}. 

\begin{theorem}
For any extended cell $u' = (i', j', \Psi')$ that is reachable from some initial extended cell $u = (i, j, \Psi_\theta)$ via a path on the mesh graph, Algorithm \ref{algorithm_dfs} will assign a number to the configuration of primitives $\Psi'$:
\begin{align*}
    &\Psi' \in \textsc{MainProcedure}()
\end{align*}
Note: All such $\Psi'$ (which are in the extended cell reachable from some initial one) are referred to as reachable.
\label{thm:number}
\end{theorem}

\begin{proof}
Let the path on the mesh graph, as mentioned in the condition, be of the form $u_1, \ldots, u_N$, where $u_l := (i_l, j_l, \Psi_l)$ for any $l \in [1, N]$. In this case, $i_1 = i, j_1 = j, i_N = i', j_N = j'$, and $\Psi_1 = \Psi_\theta, \Psi_N = \Psi'$. Note that the mesh graph is regular (i.e., invariant under parallel translation), so we can assume $i = 0, j = 0$ (we can always consider such a parallel translation to move $(i,j)$ to $(0,0)$; the successor relationship in the mesh graph does not depend on the parallel translation, which is the essence of regularity, and thus the sequence $u_1, \ldots, u_N$ will remain a path after such a shift).

Now, we will proceed by contradiction. Suppose $\Psi_N = \Psi'$ is not numbered. Consider the chain of configurations of primitives $\Psi_1, \ldots, \Psi_N$. Let $\Psi_l$ be the first configuration in this chain that has not been assigned a number. Note that $u_1 = (0, 0, \Psi_\theta)$ is one of the extended cells that the \textsc{MainProcedure} will take at line $5$, and from which it will launch \textsc{NumberingDFS} at line $6$. The latter will assign a number to the configuration $\Psi_\theta$ at line $5$, so $\Psi_1 = \Psi_\theta$ is certainly numbered, which means $l > 1$. 

In this case, we conclude that $\Psi_{l-1}$ exists and is also numbered, meaning that for some extended cell $v := (i'', j'', \Psi_{l-1})$, $\textsc{NumberingDFS}(v)$ was launched. However, in line $7$ of this procedure, all successors of $v$ must have been considered. Since $u_l = (i_l, j_l, \Psi_l)$ is a successor of $u_{l-1} = (i_{l-1}, j_{l-1}, \Psi_{l-1})$, and the successor relationship do not depend on specific coordinates due to regularity, there must be an extended cell among the successors of $v$ with the configuration $\Psi_l$. Therefore, \textsc{NumberingDFS} should have been launched from it in line $8$, which would assign it a number. This leads us to a contradiction, thus proving the theorem.
\end{proof}

Thus, the \textsc{MainProcedure} will number all configurations of primitives that may occur during the pathfinding process for trajectory construction (since, as stated in Theorem \ref{thm:main}, we only search for paths from the initial extended cell). Therefore, MeshA* can work with these numbers instead of the configurations of primitives to simplify operations, thereby accelerating the search.

\begin{table*}[ht]
    \centering
    \setlength{\tabcolsep}{0.6mm}
    \small
    \begin{tabular}{c|cccc|cccc|cccc|cccc}
        \toprule
        \multirow{2}{*}{Algorithms}& \multicolumn{4}{c|}{\texttt{Labyrinth}}& \multicolumn{4}{c|}{\texttt{ht\_0\_hightown}}& \multicolumn{4}{c|}{\texttt{Moscow\_0\_512}}& \multicolumn{4}{c}{\texttt{AR0304SR}}\\ 
        \cmidrule(lr){2-5} \cmidrule(lr){6-9} \cmidrule(lr){10-13} \cmidrule(lr){14-17} 
        & $w=1$ & $w=2$ & $w=5$ & $w=10$ & $w=1$ & $w=2$ & $w=5$ & $w=10$ & $w=1$ & $w=2$ & $w=5$ & $w=10$ & $w=1$ & $w=2$ & $w=5$ & $w=10$ \\
        \midrule
        LBA* & {$100.0$} & {$105.4$} & {$109.9$} & {$112.3$} & {$100.0$} & {$105.7$} & {$110.0$} & {$113.2$} & {$100.0$} & {$105.9$} & {$108.9$} & {$110.8$} & {$100.0$} & {$105.0$} & {$109.0$} & {$111.3$} \\
                LazyLBA* & {$100.0$} & {$105.4$} & {$109.9$} & {$112.3$} & {$100.0$} & {$105.7$} & {$110.0$} & {$113.2$} & {$100.0$} & {$105.9$} & {$108.9$} & {$110.8$} & {$100.0$} & {$105.0$} & {$109.0$} & {$111.3$} \\
                MeshA* & {$100.0$} & {$110.8$} & {$122.6$} & {$128.6$} & {$100.0$} & {$109.2$} & {$117.6$} & {$122.3$} & {$100.0$} & {$113.3$} & {$125.2$} & {$131.4$} & {$100.0$} & {$109.4$} & {$121.2$} & {$129.5$} \\
        \bottomrule

    \end{tabular}
    \caption{Median relative costs of the trajectories found as a percentage of the optimal cost.}
    \label{tab:full_costs}
\end{table*}

\subsection{Precomputed Transition Table}

With configurations numbered, we can now pre-compute the successor relationships. The mesh graph is regular (translationally invariant), meaning the successor's configuration $\Psi'$ and its displacement $(\delta_i, \delta_j)$ relative to the parent cell's position depend only on the parent's configuration $\Psi$, not its absolute coordinates $(i,j)$.

Therefore, during the numbering process in Algorithm~\ref{algorithm_dfs}, we can simultaneously build a transition table. For each configuration ID, this table stores a list of successor configuration IDs, each paired with its relative displacement and transition cost.
With this table, successor generation for an extended cell $(a, b, \text{ID}_{\Psi})$ becomes a constant-time lookup: retrieve the list of \texttt{(ID\_successor, cost, displacement)} and compute the new coordinates as $(a+\delta_i, b+\delta_j)$. This makes MeshA*'s successor generation as efficient as LBA*'s. This pre-computation is a one-time cost for a given control set, analogous to the process of generating the control set itself.

\subsection{Memory Consumption}

A potential concern is that MeshA* may consume significantly more memory than LBA*, as its search nodes (extended cells) correspond to every cell in a primitive's trace, whereas LBA*'s nodes (discrete states) only correspond to the start and end points.

This concern is addressed by a key implementation detail regarding path reconstruction. The reconstruction algorithm (Alg.~\ref{algorithm_reconstruction}) only requires the sequence of \textit{initial} extended cells from the final path. Non-initial cells are transient states within the execution of a single primitive. Their full history is not needed for the final solution, and once expanded, they do not need to be kept in memory (e.g., in the CLOSED list) as no path will ever need to be reconstructed back to them.

Therefore, only initial extended cells need to be fully stored. Since initial extended cells in MeshA* are the direct analogues of the discrete states stored by LBA*, this implementation detail ensures that the memory footprint of MeshA* is comparable to that of LBA*.

\section{Appendix C. Additional Experimental Results}

This appendix provides supplementary material to the experimental evaluation presented in the main paper, including complete solution cost data, an extended runtime evaluation on a more diverse set of maps, and visualizations of resulting trajectories.

\subsection{Complete Solution Cost Data}

Table~\ref{tab:full_costs} presents the complete data for the median solution costs across all tested maps, complementing the summary provided in Table~\ref{tab:algorithm_costs} of the main text. The cost is shown as a percentage of the optimal cost, where 100\% represents the optimal solution. As established theoretically, both LBA* and MeshA* consistently find optimal solutions ($w=1$), while their lazy and weighted variants produce solutions of varying quality.

\begin{table}[ht!]
    \centering
    \begin{tabular}{llcccc}
        \toprule
        \textbf{Map} & \textbf{Type} & \textbf{w=1} & \textbf{w=2} & \textbf{w=5} & \textbf{w=10} \\
        \midrule
        \texttt{NewYork\_1\_512} & city & 63.6 & 64.9 & 58.4 & 56.9 \\
        \texttt{Berlin\_1\_512} & city & 63.1 & 63.8 & 59.3 & 57.2 \\
        \texttt{Entanglement} & mixed & 65.3 & 72.2 & 66.9 & 65.8 \\
        \texttt{Boston\_1\_512} & city & 61.8 & 65.0 & 59.1 & 55.9 \\
        \texttt{Milan\_1\_512} & city & 61.1 & 60.9 & 56.6 & 55.3 \\
        \texttt{BigGameHunters} & mixed & 62.8 & 50.8 & 47.2 & 48.2 \\
        \texttt{London\_0\_512} & city & 61.8 & 58.5 & 53.5 & 53.3 \\
        \texttt{AR0015SR} & indoor & 62.6 & 63.4 & 56.5 & 56.6 \\
        \texttt{gardenofwar} & maze & 66.7 & 66.7 & 62.4 & 61.5 \\
        \texttt{BlastFurnace} & mixed & 61.7 & 50.3 & 49.4 & 49.7 \\
        \texttt{IceMountain} & mixed & 67.7 & 51.1 & 49.7 & 49.7 \\
        \texttt{Paris\_1\_512} & city & 64.1 & 66.4 & 58.2 & 57.6 \\
        \bottomrule
    \end{tabular}
    \caption{Median runtime of MeshA* as a percentage of LBA* runtime on additional maps. Lower values are better.}
    \label{tab:extended_runtime}
\end{table}

\subsection{Extended Runtime Evaluation on Diverse Maps}

To further validate the performance of MeshA* and demonstrate its robustness, we conducted an extended set of experiments on 12 additional maps from the MovingAI benchmark. These maps cover a wider range of environments, including city, maze, indoor, and mixed-geometry layouts. For each map, we used four heuristic weights ($w \in \{1, 2, 5, 10\}$) and ran 5,000 test instances per combination (total: $4~\text{weights} \times 5,000~\text{tests} \times 12~\text{maps} = 240,000$ instances).

Table~\ref{tab:extended_runtime} shows the median runtime of MeshA* as a percentage of the LBA* runtime. A value of 66.7\% corresponds to a 1.5x speed-up, while 50\% represents a 2x speed-up. The results consistently show a 1.5x to 2x speed-up for MeshA*, confirming the findings from the main paper across a broader and more diverse set of environments. 

The initial four maps were chosen for the main paper because they are representative of these different geometries (city, maze, mixed, etc.), and as shown here, the performance trends hold consistently within each geometry type.

\subsection{Example Trajectories}

Figure~\ref{fig:appendix_trajectories} provides a visual comparison of trajectories found with different heuristic weights, specifically chosen to illustrate a case where the suboptimality of MeshA*'s solution grows faster with the weight $w$ compared to LBA*. The optimal path ($w=1$), which is identical for both algorithms, is contrasted with a suboptimal path found by MeshA* with $w=2$. This highlights how the weighted heuristic encourages a more "greedy" search, often resulting in a path that reaches the goal faster but is less direct and has a higher cost. Both trajectories are smooth and kinodynamically feasible.

\section{Appendix D. Generalizations}

While this work focuses on the standard state lattice formulation with states $(x, y, \theta)$, the proposed MeshA* framework is generic and allows for several natural generalizations.

First, the spatial decomposition is not limited to 2D grids. The concept of "extended cells" can be directly lifted to 3D environments by replacing planar grid cells with volumetric voxels. The logic of mesh graph construction, primitive projection, and pruning rules remains conceptually identical.

Second, the framework can support arbitrary additional state parameters, provided they are discretized into a finite set of values. Indeed, our theoretical derivation relies solely on the fact that the heading $\theta$ takes values from a finite set (e.g., 16 in our experiments), ensuring a finite number of possible primitive configurations $\Psi$. This property holds for any finite tuple of non-spatial parameters. 

For instance, consider a state space augmented with curvature: $(x, y, \phi, \kappa)$, where $\phi \in \{\phi_1, \dots, \phi_n\}$ represents the discrete heading and $\kappa \in \{\kappa_1, \dots, \kappa_m\}$ represents discrete curvature. We can map every unique pair $(\phi_i, \kappa_j)$ to a single abstract variable $\Theta$ with $N = n \times m$ distinct values. By using $\Theta$ in place of $\theta$, the problem is reduced to our original formulation, making MeshA* directly applicable.

However, while theoretically sound, the practical efficiency of MeshA* in such high-dimensional spaces remains an open question. The size of the precomputed transition table and the number of extended cells grow with the complexity of the "abstract heading" $\Theta$, potentially impacting the scalability of the approach.

\begin{figure*}[ht!]
    \centering
    \includegraphics[width=1\textwidth]{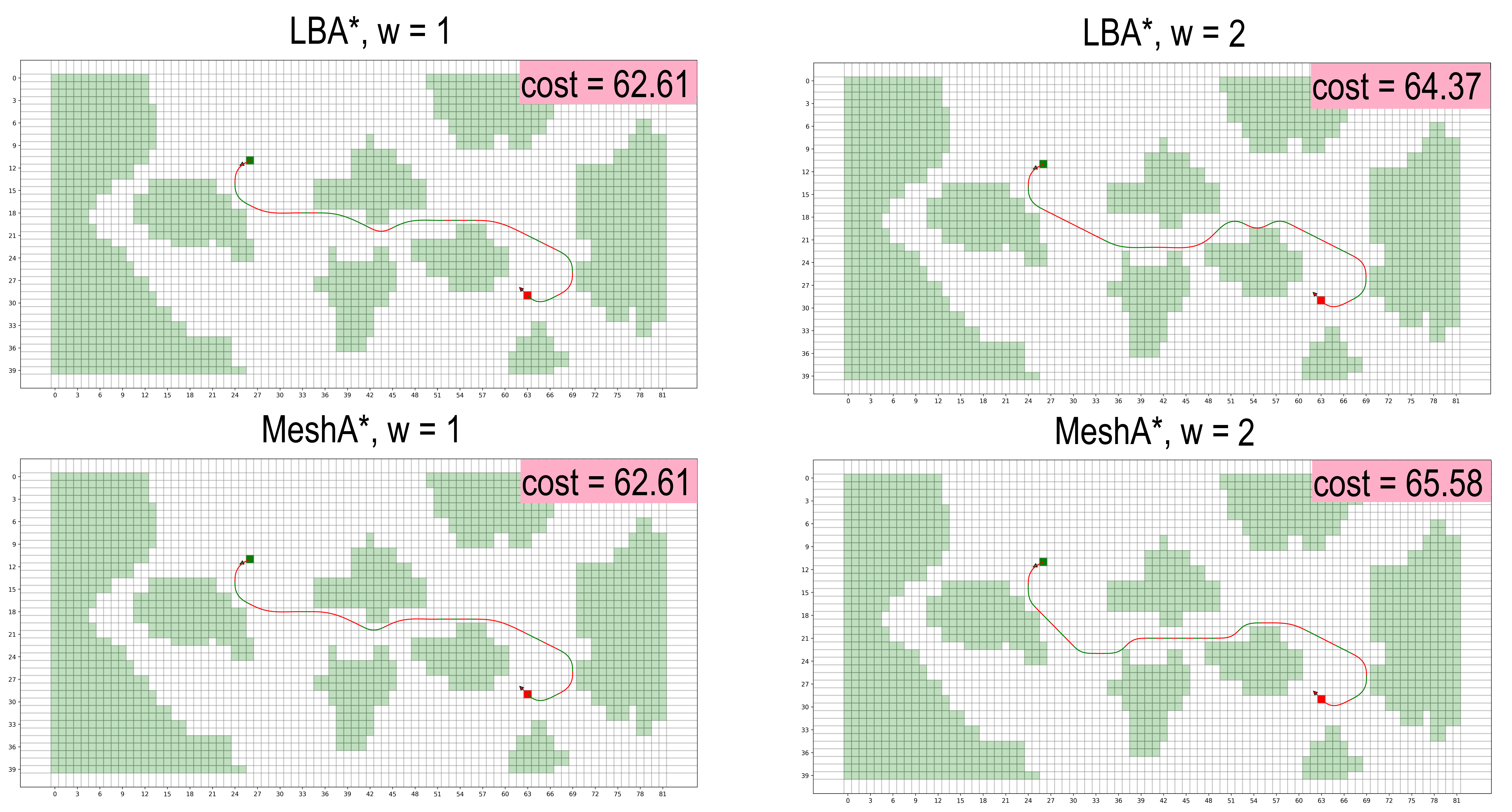}
    \captionsetup{width=0.7\textwidth}
    \caption{Visual comparison of an optimal path found with $w=1$ (left) and a suboptimal path found with $w=2$ (right) for the same start-goal pair.}
    \label{fig:appendix_trajectories}
\end{figure*}

\end{document}